\documentclass[12pt]{article}

\usepackage{setspace}

\PassOptionsToPackage{round}{natbib}

\usepackage[preprint]{neurips_2022}

\usepackage[format=plain,
            labelfont={bf},
            textfont=it]{caption}

\usepackage[utf8]{inputenc} 
\usepackage[T1]{fontenc}    
\usepackage{hyperref}       
\usepackage{url}            
\usepackage{booktabs}       
\usepackage{amsfonts}       
\usepackage{nicefrac}       
\usepackage{microtype}      
\usepackage{xcolor}         
\usepackage{array}
\usepackage{comment}
\usepackage{graphicx} 
\usepackage{amsmath,amsthm}
\usepackage{color,soul}

\usepackage{tikz}
\usetikzlibrary{bayesnet}
\usetikzlibrary{arrows}
\usepackage{color}
\usepackage{graphicx}
\usepackage{caption}
\usepackage{subcaption}
\usetikzlibrary{backgrounds}

\theoremstyle{plain}

\newtheorem{proposition}{Proposition}

\newcommand{\E}{\mathbb{E}}
\newcommand{\bX}{\mathbf{X}}

\newcommand{\bM}{\mathbf{M}}
\newcommand{\bm}{\mathbf{m}}

\title{Integrating Earth Observation Data into Causal Inference: Challenges and Opportunities}

\author{%
  Connor T. Jerzak \\
  Department of Government \\
  University of Texas at Austin \\
  Email: \texttt{connor.jerzak@austin.utexas.edu} \\
  Website: \texttt{ConnorJerzak.com} \\
   \And
    Fredrik Johansson  \\
   Data Science and AI Division \\
   Chalmers University of Technology \\
   Email: \texttt{fredrik.johansson@chalmers.se} \\
   Website: \texttt{fredjo.com}
   \AND
   Adel Daoud \\
  Institute for Analytical Sociology\\
  Linköping University \\
  Email: \texttt{adel.daoud@liu.se} \\
  Website: \texttt{AdelDaoud.se} \\ 
  AI and Global Development Lab: \texttt{global-lab.ai}
}

\renewcommand{\S}{Section\;}

\begin{document}
\maketitle

\begin{abstract}
Observational studies require adjustment for confounding factors that are correlated with both the treatment and outcome. In the setting where the observed variables are tabular quantities such as average income in a neighborhood, tools have been developed for addressing such confounding. However, in many parts of the developing world, features about local communities may be scarce. In this context, satellite imagery can play an important role, serving as a proxy for the confounding variables otherwise unobserved. In this paper, we study confounder adjustment in this non-tabular setting, where patterns or objects found in satellite images contribute to the confounder bias. Using the evaluation of anti-poverty aid programs in Africa as our running example, we formalize the challenge of performing causal adjustment with such unstructured data---what conditions are sufficient to identify causal effects, how to perform estimation, and how to quantify the ways in which certain aspects of the unstructured image object are most predictive of the treatment decision. Via simulation, we also explore the sensitivity of satellite image-based observational inference to image resolution and to misspecification of the image-associated confounder. Finally, we apply these tools in estimating the effect of anti-poverty interventions in African communities from satellite imagery.

{
\small 
\vspace{.10in}
 
\vspace{.05in}
\noindent {\bf Word count: } 12,207

\vspace{.05in}
\noindent {\it Note: } This work largely subsumes 
\begin{itemize}
\item[] Jerzak, Connor T., Fredrik Johansson, and Adel Daoud. ``Estimating Causal Effects Under Image Confounding Bias with an Application to Poverty in Africa.'' {\it arXiv preprint} \url{arXiv:2206.06410} (2022).
\end{itemize}

} 
\end{abstract}

\newpage 
\section{Introduction}\label{ss:intro}
\vspace{-0.5cm}
The causal revolution in the social sciences has entered a new phase where scholars are increasingly combining traditional tabular data with novel data sources \citep{daoud_statistical_2023, morgan2015counterfactuals,imai2022causal,pearl_causality_2009}. And, as a response to increasing data digitization and availability, a growing literature has emerged more recently that seeks to provide valid estimation of treatment effects in the presence of high-dimensional confounders, where the number of variables is large relative to the number of observations. \citep{li2021bounds,mozer2020,chernozhukov2018double,yoon2018ganite,wager2018estimation,shalit2017estimating,hill2011bayesian,schneeweiss2009high, Belloni2014, alexanderOverlapObservationalStudies2021}. Yet, merely adding more covariates describing each unit's context does not itself address concerns about non-random missingness within that confounder data, about the difficulty in obtaining information, about historical interventions, and about the lack or unreliability of data in the most economically disadvantaged places \citep{jerven2015statistical,dinku2019challenges}---precisely those in need of effective interventions. 

In this article, we argue that satellite data and remote sensing information provide an important resource for expanding the reach of causal inquiry in otherwise data-scarce environments. Satellite data is available for every corner of the globe. This breadth of information is available on a historical basis stretching back to the 1970s. Moreover, unlike a sizable proportion of datasets which are a single snapshot of a social system at a particular point in time, many earth observation satellites return to every place on earth every two weeks or even more frequently---providing 26 or more temporal slices per year. The temporally-resolved  information contained in satellite data has been shown to be associated with variables of social science importance often characterized as neighborhood features---features such as the development of transportation networks \citep{nagne2013transportation}, the degree of urbanness \citep{schneider2009new}, health and material conditions \citep{daoud2021using}, living standards \citep{jean2016combining,yeh2020using}, and a host of other neighborhood features \citep{sowmya2000modelling}. For these reasons, data from space-borne instruments may provide critical information for causal inference analyses, especially given the rapid proliferation of earth observation satellites from the hundreds into the thousands \cite{tatem2008fifty} and sub-100 cm resolution now widely available on the latest generation of satellites \citep{hallas2019mapping}.

Despite the potential offered by satellite images in causal inference, there is a lack of methodological guidance for causal estimation when confounding is induced by patterns or objects observed in an image  \citep{castroCausalityMattersMedical2020}. To help fill this need, we examine observational causal inference in the presence of confounding captured as latent image patterns. A story-based intuition for we have in mind is the following: an actor examines a neighborhood, looking for certain aspects of that neighborhood (such as the presence of poverty) to guide the choice of intervention, $T$. Those neighborhood aspects, $U$, are observed by the decision maker but unobserved by outside analysts. Some of that neighborhood-level information is also embedded in the satellite data representation ($\bM$) of that same neighborhood. Thus, observed image patterns indicate the existence of real-world objects that provide information about the confounding factors associated with both treatment and outcome. We may not directly observe the true latent variables about which the treatment decision was made, but we can readily observe, from earth observation resources, and adjust for inferred image patterns that correlate with the treatment, even if these patterns are difficult to adjust for directly \citep{voigtGlobalTrendsSatellitebased2016}.

Our focus in this paper on causal inference with earth observation data complements a social-science research trend, especially in sociology and political science, where scholars increasingly leverage visual data. For example, such visual data are used in qualitatively analyzing photos \citep{pauwels_visual_2010,ohara_participant_2019}, estimating image similarities \citep{zhang_image_2022}, and approximating the number of demonstrating people from news photos \citep{cruz_unbiased_2021}. Recently, the use of visual data also includes video data for analyzing social processes, such as police violence \citep{nassauer_video_2021}. For larger visual data and in a quantitative design, scholars have to train algorithms ``learning to see'' objects of interest \citep{torres_learning_2022}. Nonetheless, although these contributions are critical for research designs relying on image data, there is a need for deeper grounding in the causal inference literature \citep{daoud_statistical_2023, morgan2015counterfactuals,imai2022causal,pearl_causality_2009}, creating a knowledge gap about how to leverage images for causal inference. Our article contributes to filling that gap. 

In what follows, we first in \S\ref{ss:theory} describe several causal structures relevant to satellite image-based causal inference and discuss their implications for identification and estimation. We focus on under what conditions the image alone is sufficient to adjust for the confounding introduced by $U$. This holds, for example, when $U$ may be derived deterministically from $\bM$. An important special case of image pattern confounding occurs when decision-makers make choices based on the (translation-invariant) existence of a pattern in the image, which motivates adjustment techniques based on convolutional models developed in the machine learning community \citep{goodfellow_deep_2016}. We also analyze the complementary case where the image confounder is itself the cause of the image.

We next study in \S\ref{ss:experiments}  finite-sample estimation of average treatment effects in the fully identified case by conducting a simulation in which confounders are derived from the observed image. In this setting, we investigate the impact of model misspecification on estimates, as well as the role of image resolution---a key aspect of image data with no perfect analog in tabular, network, or text data. Finally, in \S\ref{sec:mu_empirical}, we demonstrate the use of the proposed estimation framework in an application in which we evaluate the effectiveness of international aid programs on neighborhood-level poverty by estimating treatment propensity using geo-referenced satellite data to proxy for confounding factors.

\section{Related Work and Contribution} 
Observational causal inference methods primarily evolved in the context of tabular covariates (e.g., see \citet{imbens2016causal}), where separate, human understandable features are used as covariates in adjusting for confounding via regression \citep{best2013sage}, weighting \cite{jung2020estimating}, or doubly robust \citep{funk2011doubly} methods. While tabular covariates are readily interpretable and can be tailored by researchers at the time of data collection, they face several limitations. For example, tabular covariates such as gender, income, and ethnicity are typically collected at the time of data collection by researchers. Thus, they cannot capture historical confounders (such as income pre-data collection) without additional effort. Moreover, they are subject to missingness or mismeasurement due to unit-level behavior: a unit may decide--or not---to answer a survey question or may mask their true opinions, both of which induce bias into the resulting causal analysis. 

The causal text analysis literature has recently developed interesting insight into how text can supplement tabular covariates in observational causal inference, emphasizing how text information can provide insight into why treatments were assigned \citep{grimmer2013text,egami2018make,keith2020text}, even if there is a risk posed by the open-ended quality of text that can lead to bias in observational analyses \cite{daoud2022conceptualizing_n}. While text does pose an opportunity for investigators, it is limited by availability pre-intervention: much text is gathered via web resources, so historical interventions can be difficult to model via text, not to mention the fact that some interventions might not have relevant text associated with them. While text is a promising new data source for causal inference, it is not available universally for all spatially defined interventions. 

Social network data has also been increasingly emphasized in the causal inference setting. The typical use of this network data is not to adjust for confounding bias but instead to model how units respond to the treatment status of social connections, primarily to estimate spillover effects \cite{vanderweele2013social}. Network data must be gathered for every unit about every other unit: this kind of data source scales non-linearly---with a polynomial growth in the number of possible connections as the number of units. As a consequence, network data has generally been difficult to obtain outside a few contexts such as trade relations, alliances, or online social networks, where social links are assumed to be online website links (e.g., \citet{lewis2012social}). Because of the computational complexity, the temporal resolution of network data used in social science research is generally refined at best on a yearly basis, as in international relations network analyses (see, for example, \citet{lebo2008dynamic}). 

In this context, satellite image data potentially addresses a number of the limitations of other data streams for observational causal inference (see Table~\ref{tab:DataComparison} for a summary). Satellite image data--- obtained from space-born instruments funded by NASA, the European Space Agency, and others---is high-dimensional and readily interpretable as a raw data source in that human observers can generally identify key features of the data, such as the presence of a city, forest, river, soil moisture content, and so forth. Satellite data are also available for all geo-referenced interventions---not just those about which a human observer decided to write a document. Moreover, while satellite data does have missingness due to cloud cover, experimental units cannot affect the availability of satellite images about their neighborhood, an inability that reduces the risk of bias due to systematic data missingness \citep{kenward2007multiple}. Finally, remote sensing data also has a temporal resolution of two weeks or better for many leading satellite image providers---meaning that, in principle, one can examine the evolution of a neighborhood with 26 or more data slices per year. 

\begin{table}[ht]
\centering
\renewcommand\arraystretch{1.5}
\caption{Comparing satellite imagery data with other data types in the context of observational causal inference. $*$ The European Space Agency's Sentinel earth observation satellites revisit locations at least once every five days; Landsat revisits locations around once every 14 days.}\label{tab:DataComparison}
\begin{tabular}[t]{
>{\raggedright}p{0.23\linewidth}
>{\raggedright}p{0.15\linewidth}
>{\raggedright}p{0.10\linewidth}
>{\raggedright\arraybackslash}p{0.15\linewidth}
>{\raggedright\arraybackslash}p{0.15\linewidth}
>{\raggedright\arraybackslash}p{0.35\linewidth}
}
\toprule\midrule 
& {\it Tabular} & {\it Text} & {\it Network} & {\it  Earth  Observation}
\\ \midrule
\underline{\it Origin \& Character} &  &  &  &
\\ Common origin & Researcher surveys/gov't data & Web resources & Online social network databases & Space-born instruments
\\ High-dimensional & Not usually & Yes & Yes & Yes
\\ Unstructured in raw form & No & Yes & Yes & Yes
\\ Readily interpretable as a raw data source & Yes & Yes & No & Yes
\\ Scaling with \# of units & Linear & Linear & Polynomial & Linear
\vspace{0.4cm}
\\ \underline{\it Availability} &  &  &  & 
\\ Available for all geo-referenced interventions & No & No & No & Yes
\\ Readily available pre-intervention & No & No & No & Yes
\\ Typical temporal resolution & Singular/1 year & Singular & Singular & 5-14 days$^*$
\\ Subject to missingness due to unit behavior & Yes & Yes & Yes & No
\\ \bottomrule
\end{tabular}
\end{table}%

The machine learning community has begun exploring the opportunities of images in the causal image context \citep{castroCausalityMattersMedical2020, reddyCANDLEImageDataseta, duAdversarialBalancingbasedRepresentation2021}. Related work has been done on estimating counterfactual outcomes under spatially defined counterfactual treatment strategies  \citep{papadogeorgou2020causal}, on accounting for spatial interdependence in causal effect estimation \citep{reich2021review}, and on balancing observed covariate representations using adversarial networks and image data  \citep{kallus2020deepmatch}. Other works address images as treatments \citep{kaddour2021causal}, counterfactual inference and interpretability \citep{pawlowskiDeepStructuralCausal2020, singlaExplanationProgressiveExaggeration2020}, and image-based treatment effect heterogeneity \citep{jejoda2022_hetero}. There are also several important works on how algorithms may discover causal structures from images in the causal discovery tradition \citep{chalupkaMultiLevelCauseEffectSystems2016,chalupkaUnsupervisedDiscoveryNino2016,chalupkaVisualCausalFeature2015,scholkopfCausalRepresentationLearning2021,yiCLEVRERCoLlisionEvents2020,dingDynamicVisualReasoning2021}. 

What we can add to this literature is twofold. First, we provide a systematic analysis on how to understand the latent confounding structure proxied for or induced by images in the causal inference setting, where latent objects in the image may affect treatment and outcome \citep{castroCausalityMattersMedical2020}. Developing techniques for causal inference under image pattern confounding could open up new avenues for observational studies \citep{pawlowskiDeepStructuralCausal2020, singlaExplanationProgressiveExaggeration2020, kaddour2021causal}. (See \S\ref{ss:proxy} for a more expansive connection developed with the proxy literature.)

Second, this analysis also specifically provides guidance for social scientists interested in using earth observation resources to perform observational causal inference in neighborhood or developing contexts. As mentioned, earth observation data provides a large amount of raw information about potential confounding variables of importance in social science research: we here analyze the promise and pitfalls of this emerging analytic pipeline. 

To conclude this section, we also note that this work on integrating satellite images into causal pipelines in data-poor environments also has policy and practical resonance. Policymakers may consult maps to evaluate where to allocate aid to villages that otherwise may remain poor \citep{holmgrenSatelliteRemoteSensing1998a,bediMorePrettyPicture2007}. For example, since 2000, policymakers frequently rely on raw satellite images to evaluate damage due to natural disasters or war \citep{voigtGlobalTrendsSatellitebased2016, burkeUsingSatelliteImagery2021, kinoScopingReviewUse2021}.  
Based on these images, they decide where to intervene in helping the poor \citep{borieMappingResilienceCity2019}. Thus, the use of satellite images in causal inference may help those in the policy world as well as evaluate anti-poverty programs, especially since areas experiencing poverty are also often affected by weak state capacity \citep{besley2010state}, meaning that the poorest places are often those about which we know least in the absence of remote sensing information. 

\section{Characterizing Challenges and Opportunities of Satellite Image-based Observational Inference }\label{ss:theory}
A major reason why satellite images can serve as a useful tool for observational causal inference is that the satellite images can proxy for confounders that explain why some places but not others received a treatment of some kind, such as an anti-poverty aid program. 

With this motivation in mind, we study the identification and estimation of the average treatment effect (ATE) of $T$ on a real-valued outcome of interest, $Y$, based on observational data (suppressing the unit subscripts of the various random variables except when important for expository purposes). With $Y(1)$ being the potential outcome~\citep{rubin2005causal} under intervention with $T=1$ and $Y(0)$ for $T=0$, 
$$
\text{ATE} = \E[Y(1) - Y(0)]~.
$$
The ATE can represent, for example, the difference in average wealth in villages after anti-poverty interventions and after no intervention, respectively. In that setting, historical interventions can be thought of as being determined by a decision-maker using information about neighborhoods that are proxied by a satellite image representation, $\bM$. 

Whether we think that the image patterns cause the confounder or the confounder dynamics cause the image pattern is of little consequence in practice: in the latter, the image is \textit{proxy}; in the former, the image is a  \textit{driver}; the same type of analysis will be done in either case \citep{pearlLinearModelsUseful2013} as described in the proxy literature \citep{kurokiMeasurementBiasEffect2014,louizosCausalEffectInference2017a}. To understand the estimation dynamics of the ATE from observational images and adjust for induced confounding bias, we analyze the data-generating process next.

\subsection{Baseline Model of Confounding Bias} 
To understand similarities and differences of observational causal inference in the tabular case and in the satellite image context, first consider the causal graph in Figure~\ref{fig:SimpleDag}. Note that we reintroduce unit-level subscripts in the Figure to emphasize some distinctions: $s$ denotes a scene (e.g., village), and $w$ and $h$ denote the width and height indices of a spatially resolved context like an image. This figure depicts classical confounding: a treatment of interest ($T_{swh}$), such as an anti-poverty intervention, is associated with factors (such as the presence of mineral extraction sites) that affect both the treatment and the outcome ($Y_{swh}$). Observed confounding variables are grouped in $X_{swh}$; unobserved confounders in $U_{swh}$.
\begin{figure}[t]
    \begin{subfigure}{0.48\textwidth}
        \centering
        \vspace{1em}
        \begin{center}
        \tikzstyle{main node}=[circle,draw,font=\sffamily\small\bfseries]
        \tikzstyle{sub node}=[circle,draw,dashed,font=\sffamily\small\bfseries]
        \begin{tikzpicture}[->,>=stealth',shorten >=1pt,auto,node distance=2cm,thick]
          
          \node[sub node] (1) {$U_{swh}$};
          \node[main node] (3) [below right of=1] {$T_{swh}$};
          \node[main node] (4) [right of=3] {$Y_{swh}$};
          \node[main node] (5) [right of=1] {$X_{swh}$};
    
          \path[every node/.style={font=\sffamily\small}]
          (1) edge node  {} (3) 
          (1) edge node  {} (4) 
          (3) edge node  {} (4)
          (5) edge node  {} (3) 
          (5) edge node  {} (4); 
        \end{tikzpicture}
      \end{center}
      \vspace{2.5em}
      \caption{Baseline causal diagram. }\label{fig:SimpleDag}
  \end{subfigure}
  \hfill
  \begin{subfigure}{0.48\textwidth}
        \centering
        \tikzstyle{main node}=[circle,draw,font=\sffamily\small\bfseries]
        \tikzstyle{sub node}=[circle,draw,dashed,font=\sffamily\small\bfseries]
        \begin{tikzpicture}[->,>=stealth',shorten >=1pt,auto,node distance=2cm,thick]
          
          \node[sub node] (1) {$U_{swh}$};
          \node[main node] (2) [below of=1] {$M_{sw'h'}$};
          \node[main node] (3) [below right of=1] {$T_{swh}$};
          \node[main node] (4) [right of=3] {$Y_{swh}$};
          \node[main node] (5) [right of=1] {$X_{swh}$};
          
          \node[rectangle,draw=gray, fit=(2),inner sep=2mm,label=below:{$w',h'\in \Pi_s(wh)$}] {};

          \path[every node/.style={font=\sffamily\small}]
          (1) edge node  {} (3) 
          (1) edge node  {} (4) 
          (2) edge node  {} (1) 
          (1) edge node  {} (2) 
          (3) edge node  {} (4)
          (5) edge node  {} (3) 
          (5) edge node  {} (4); 
        \end{tikzpicture}
        \caption{Diagram illustrating image-based confounding. }\label{fig:ConvDag}
    \end{subfigure}  
    \caption{Diagram representing variables associated with a scene $s$. In our running example, $U_{swh}$  represents unobserved confounders, $X_{swh}$ observed confounders, $T_{swh}$ treatment and $Y_{swh}$ the outcome, all at location $w,h$ in scene $s$. In the right-hand diagram, latent confounders $U_{swg}$ are determined by a neighborhood $\Pi_s(wh)$ of the location $h, w$ in the image $M$ representing scene $s$. The arrow between $U_{swh}$ and $M_{sw'h'}$ is bi-directional to indicate that we are agnostic about the direction of causality (whether the image ``causes'' the confounder or the confounder ``causes'' the image).}
\end{figure}
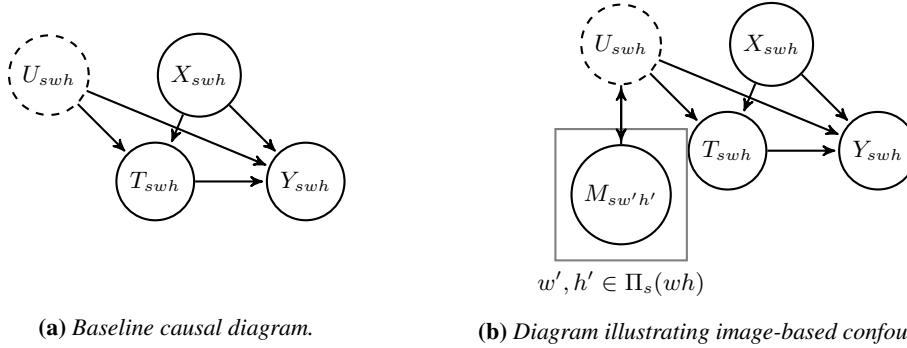
In the general, $s$ need not be spatially defined when working with non-satellite-derived images, like if the images were to be X-ray snapshots. In that case, the scene index would refer to a patient or body part. However, in the social science setting with satellite data, $s$ will usually refer to some neighborhood or the neighborhood context around an observational unit. 

\subsection{Pixel-based Image Confounding}
We now turn to the causal model in Figure~\ref{fig:SimpleDag}, which depicts the kind of confounding dealt with in much of observational research \citep{rosenbaum1983central}. We extend this model to describe image-based confounding. 

First, we define this model at the local level, where treatments are implemented at specific locations $w,h$ (in, for example, the precision agriculture context where treatments are applied to small sections of land; \citet{liaghat2010review}). 

To describe causal dependencies at the neighborhood level, we introduce the following notation. Let $\Pi_s(wh)\in \mathbb{N}^{2}$ denote the set of location indices involved in generating the confounder value from the neighborhood around $swh$. For example, if the confounder at $swh$ were generated using neighborhood information in a $z\times z$ region around $swh$, 
\begin{equation*}
{
\Pi_s(wh) =  
\big\{w-\lfloor z/2\rfloor,...,w+\lfloor z/2\rfloor\big\} 
\times 
\big\{h-\lfloor z/2\rfloor,...,h+\lfloor z/2\rfloor\big\}. 
}
\end{equation*}
For example, if $z=2$, $\Pi_s(wh) =  \{ w - 1, w, w+1\} \times \{ h - 1, h, h+1\}$, where the $\times$ symbol here denotes the Cartesian product capturing all ordered pairs of the left and right set. In other words, $\Pi_s(wh)$ simply characterizes the set of width and height indices centered around location $wh$ in scene $s$.

With this notation, we can illustrate the confounding structure induced by image-based confounding at the neighborhood level in Figure~\ref{fig:ConvDag}. Figure \ref{fig:ConvDag} is a formulation of spatial interdependence, as the image-information for indices in $\Pi_s(wh)$ affects the confounder $U_{swh}$. Conditional on the value of the confounder, the treatment decision for each unit is made. This confounding dynamic would be invoked if a decision maker who scans a scene looking for similarity of the neighborhood around $swh$ to some mental image, defined, for example, by an image filter $l$, and makes a decision on this basis. 

We illustrate this process on satellite image data from Landsat, a U.S. Geological Survey and NASA satellite (see \S\ref{ss:data} for details). We perform the illustration using a particular parametric model for the neighborhood-induced confounding. The parametric model used for illustration is convolutional: we let a single convolutional filter in the form of a diagonal matrix represent an image pattern used by a decision maker to determine the treatment probability, as depicted in Figure \ref{fig:conv}. After applying $f_l$ to the raw image shown in the right panel of Figure \ref{fig:conv}, we obtain the resulting image-derived confounder values. ``Applying the filter to the image'' here means calculating a similarity score at every place in the original image with the filter (the ``image pattern''). This score generates a new, latent image representing the similarity structure which, here, underlies the confounding dynamic. The similarity score is calculated by summing up the result of multiplying the filter with the relevant image pixels (similar to how the leading term of the covariance between $A$ and $B$ would be calculated by taking the average of the product of $A$ and $B$). This simple example shows how the presence of objects or patterns in images (as represented by the diagonal line here) can generate confounder values in the context of satellite-based observational inference. 

\begin{figure}[t]
    \centering
    \begin{subfigure}[t]{.3\textwidth}
        \centering
        \includegraphics[width=.8\textwidth]{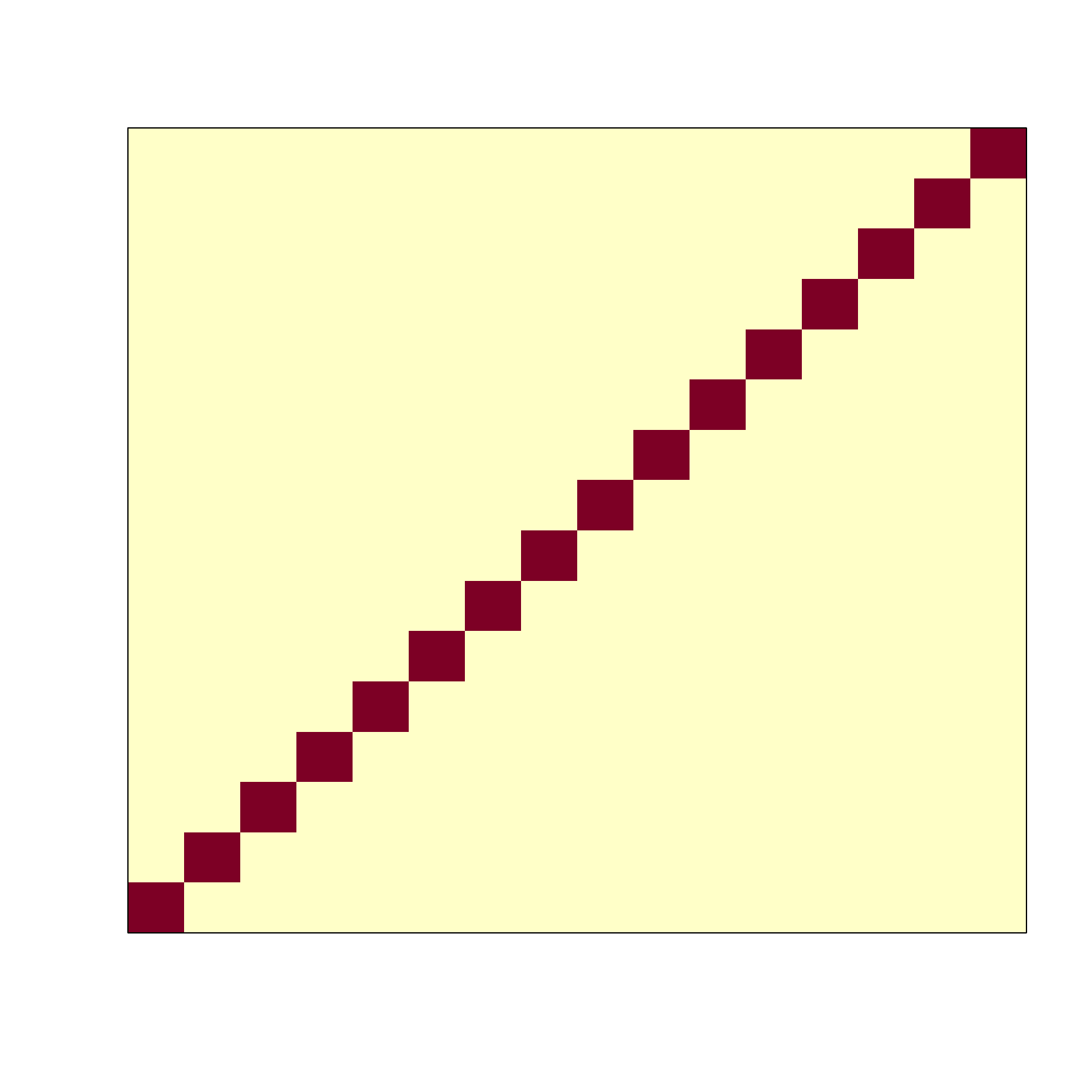}
        \vspace{-1em}
        \label{fig:convKern}
    \end{subfigure}
    \hfill
    \begin{subfigure}[t]{0.65\textwidth}
        \centering
         \includegraphics[width=.85\textwidth]{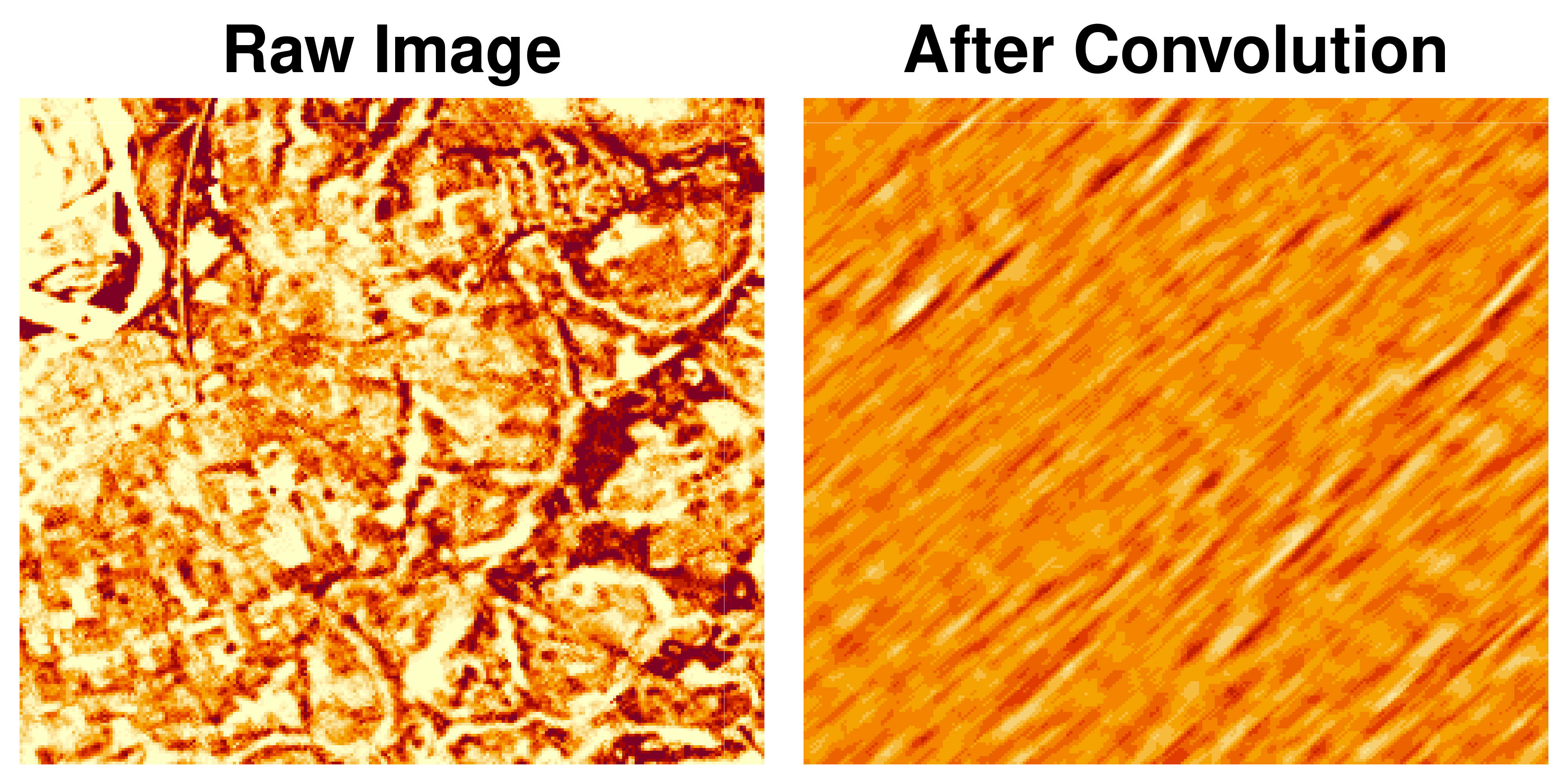}
    \end{subfigure}
        \caption{\emph{Left.} The kernel filter pattern used to generate $U_{swh}$ in the right-hand image. \emph{Center, right.} Illustration of image-based confounding using Landsat data for Nigeria. The center panel depicts the raw reflectance; the right panel depicts the transformed values after convolution with the filter, values which enter the model for treatment/outcome to generate confounding.}\label{fig:conv}
\end{figure}

\subsection{Scene-based Image Confounding}
While the pixel-based confounding structure may be relevant if small sections of land (e.g., houses) receive treatment, in many social science applications, the image is defined at one resolution, but the treatment, outcome, and confounder potentially at another. In the aid context, the treatment and outcome data are often defined at the neighborhood or village level, but the satellite data itself contains additional height and width dimensions (in a word, we can ``peek inside'' the village). To take another example, a policymaker may examine an entire village, looking for the maximum or average similarity to some target pattern: the village is the unit to which treatment is allocated, and the confounder is defined at that level but is created using more granular information. 

Satellite image-based observational inference can accommodate situations such as these, where the confounder, treatment, and outcome are defined at different scales. In particular, we can add the following to our causal model. Let $\Pi_s \in \mathbb{N}^{2}$ denote the height and width indices (locations) used when aggregating up information to the final scene-based unit of analysis, $s$. The right panel of Figure \ref{fig:conv} then illustrates a scene-based confounding structure generated by $\Pi_s$, where $\Pi_s$ intuitively represents the image indices associated with the whole neighborhood context that was used in deciding upon treatment. 

To further explain the role of $\Pi_s$, consider how, if investigators are interested in a village, $s$, the width of the satellite image collected around $s$ fixes the maximum size of $\Pi_s$ that can be accommodated in the estimation procedure. If the image is, say, 1024 by 1024 pixels, then the cap on $\Pi_s$ is $\{1, ..., 1024\}\times \{1,...,  1024\}$, which specifies a square of 14.59 km in width using the Landsat images used throughout this paper. In short, the maximum neighborhood around the village that can be used to recover the latent confounder is fixed by the width of the satellite image input. 

In general, this scene-based causal model is significant because, while some studies are able to obtain resolved (e.g., household-level) outcome data, this data may, in other cases, be costly or even impossible to obtain due to privacy reasons. In these situations, treatment and outcome data can only be measured at the scene level. We next turn to the question of identification in this image-confounding context. 

\begin{figure}[t]
\begin{subfigure}{.48\textwidth}
    \vspace{2.8em}
    \begin{center}
    \tikzstyle{main node}=[circle,draw,font=\sffamily\small\bfseries]
    \tikzstyle{sub node}=[circle,draw,dashed,font=\sffamily\small\bfseries]
    \begin{tikzpicture}[->,>=stealth',shorten >=1pt,auto,node distance=1.8cm,thick]
      
      \node[sub node] (0) {$U_{swh}$};
      \node[sub node] (1) [right of=0] {$U_s$};
      \node[main node](2)[below of =0, yshift=2mm]{$M_{sw'h'}$};
      \node[main node] (3) [below right of=1] {$T_s$};
      \node[main node] (4) [right of=3] {$Y_s$};
      \node[main node] (5) [right of=1] {$X_s$};
      
      \node[rectangle,draw=gray, fit=(0) (2),inner sep=7mm,label=below:{$w,h\in \Pi_{s}$}] {};
      
      \node[rectangle,draw=gray, fit=(2),inner sep=1mm,label=below:{\tiny $w',h'\in$ \newline $\Pi_s(wh)$}] {};

      \path[every node/.style={font=\sffamily\small}]
      (1) edge node  {} (3) 
      (1) edge node  {} (4) 
      (2) edge node  {} (0) 
      (0) edge node  {} (1) 
      (1) edge node  {} (0) 
      (3) edge node  {} (4)
      (5) edge node  {} (3) 
      (5) edge node  {} (4) ;
    \end{tikzpicture}
  \end{center}
  \caption{Image-based confounding at the scene level.}\label{fig:ComplexConvDag}
 \end{subfigure}
 \hfill
 \begin{subfigure}{.48\textwidth}
    \begin{center}
    \tikzstyle{main node}=[circle,draw,font=\sffamily\small\bfseries]
    \tikzstyle{sub node}=[circle,draw,dashed,font=\sffamily\small\bfseries]
    \begin{tikzpicture}[->,>=stealth',shorten >=1pt,auto,node distance=1.8cm,thick]
      
      \node[sub node] (0) {$U_{swh}$};
      \node[sub node] (1) [right of=0] {$U_s$};
      \node[sub node] (6) [above of=1] {$R_s$};
      \node[main node](2)[below of =0, yshift=2mm]{$M_{sw'h'}$};
      \node[main node] (3) [below right of=1] {$T_s$};
      \node[main node] (4) [right of=3] {$Y_s$};
      \node[main node] (5) [right of=1] {$X_s$};
      
      \node[rectangle,draw=gray, fit=(0) (2),inner sep=7mm,label=below:{$w,h\in \Pi_{s}$}] {};
      
      \node[rectangle,draw=gray, fit=(2),inner sep=1mm,label=below:{\tiny $w',h'\in$ \newline $\Pi_s(wh)$}] {};

      \path[every node/.style={font=\sffamily\small}]
      (1) edge node  {} (3) 
      (6) edge node  { Objects not in image} (1) 
      (1) edge node  {} (4) 
      (2) edge node  {} (0) 
      (0) edge node  {} (1) 
      (1) edge node  {} (0) 
      (3) edge node  {} (4)
      (5) edge node  {} (3) 
      (5) edge node  {} (4);
    \end{tikzpicture}
  \end{center}
  \caption{Some confounding not observed in the image.}\label{fig:ResidualBias}
  \end{subfigure}
  \caption{}
\end{figure}

\subsection{Identification in Satellite Image-based Observational Inference}\label{ss:identification} 
We can confirm that under image-based confounding as formalized in Figures \ref{fig:ConvDag} and \ref{fig:ComplexConvDag}, treatment effects may be identified by adjusting for the image $\bM$. Here, we suppress dependence on the indexes $s,w,$ and $h$ since the same arguments will apply if $T,U,$ or $Y$ are defined at the $swh$ (pixel) or $s$ (scene) levels.

To begin, we assume that the confounder $U$ is a deterministic function of $\bM$ and return to the case where $U$ has multiple causes later. This is justified, for example, in applications where confounding is based on the existence of an object---either if the policymaker scanned $\bM$ for the object prompting the  policymaker to allocate a treatment in that area of interest, or if $\bM$ is more generally associated with important confounder patterns/objects without error.

As $U$ is determined fully by $\bM$, ruling out other potential noise sources, there exists a deterministic function $f$ such that $U = f(\bM)$. The aforementioned case of $U$ being the (pooled) convolution of a 2D filter with the image $\bM$ satisfies this assumption. 

\begin{proposition}\label{prop:identification}
Suppose the confounder $U$ is deterministic given the image $\bM$, such that $U = f(\bM)$, (with $f$ unknown), and that the causal structure obeys either of Figures \ref{fig:ConvDag} \& \ref{fig:ComplexConvDag}. Then, $p(Y(t))$ and therefore ATE of $T$ on $Y$ is identifiable from $p(\bM, X, T, Y)$.
\end{proposition}
\begin{proof} For simplicity of exposition, we give the proof for the case without additional  confounding variables $X$. The proof generalizes readily to non-empty $X$, by marginalization and conditioning. The claim follows from $U$ being a deterministic function of $\bM$. By the backdoor criterion applied to the graphs in Figures \ref{fig:ConvDag} \& \ref{fig:ComplexConvDag}, $X, U$ is an adjustment set for the effect of $T$ on $Y$, which implies the exchangeability of potential outcomes: $Y(t) \perp T \mid X$ (see, e.g., \citet{hernanbook}). In the case of empty $X$,
\begin{equation}
p(Y(t)) = \sum_{u} p(Y|T=t,U=u)p(U=u).
\end{equation}
Since $U$ is a deterministic, but not necessarily invertible function of $\bM$, $U=f(\bM)$, we have that 
\begin{align}%
p(Y\mid T=t, U=u) &= p(Y\mid T=t,\bM\in f^{-1}(u)) \;\;\mbox{ and }\;\;
p(U=u) &= \sum_{\bm \in f^{-1}(u) } p(\bM=\bm) \label{eq:marginal-u}
\end{align}%
where $f^{-1}$ is the inverse map of $f$, so that 
\begin{align*} 
p(Y(t)) &= \sum_u \sum_{\bm \in f^{-1}(u) } p(Y|T=t,\bM=\bm)p(\bM=\bm) = \sum_{\bm} p(Y|T=t,\bM=\bm)p(\bM=\bm)~.
\end{align*} 
Hence, $\bM$ is also an adjustment set for $T$ on $Y$. From a similar proof, we see that $X, \bM$ is an adjustment set in the case of non-empty $X$. From here, standard arguments \citep{rosenbaum1983central} show that
\begin{align}
\mathbb{E}[Y(t)] = \mathbb{E}\left[Y \frac{p(T=t)}{p(T=t\mid \bM=\bm)} \ \bigg| \ T=t \right]
\end{align}
which justifies the use of inverse-propensity weighting with respect to $\bM$.
\end{proof}

The argument of Proposition~\ref{prop:identification} rests on the assumption that the image contains all information about the latent confounder. When treatment decisions are made based on object detection, this assumption would be satisfied if the image contains all objects that are relevant to the outcome and treatment decision. This is violated if, for example, unlabeled objects, depicted as $R_s$ in Figure \ref{fig:ResidualBias}, are themselves the driver of the treatment decision but are not possible to reconstruct from the image data. If the image data imperfectly depicts those objects, full identification is no longer possible, as there is a possibility of residual confounding. Specifically, the inverse map $f^{-1}$ in Equation \ref{eq:marginal-u} is no longer uniquely or well defined as a set of images; $R_s$ must be adjusted for as well. In this imperfect case, the image becomes a \textit{driver} of the confounding, and thus, has similar properties to proxies \citep{pearlLinearModelsUseful2013,penaMonotonicityNondifferentiallyMismeasured2020a}.

We have shown how, under assumptions, the image is itself an adjustment set for estimating the effect of programs on outcomes in the context of image-based confounding. However, non-parametric inference is difficult in the image context because no two images are the same. Thus, the probability of seeing identical treatment/control images is zero, violating overlap assumptions necessary for model-free inference  \citep{d2019comment}. Machine learning models for the image may seek to estimate $U$, forming latent representations for the image. In this lower-dimensional space, there is more likely to be empirical overlap between treatment/control, justifying the use of modeling approaches like the ones discussed in \S\ref{ss:experiments} and \S\ref{sec:mu_empirical}. Thus, while adjusting directly for $U$ would fulfill the overlap assumptions optimally, this is infeasible; when adjusting for $\bM$ instead, a critical argument for our approach to work is that the propensities depend only on aspects of $\bM$ that capture $U$, aspects assumed to be compressible to a lower dimensional representation. As such, the situation is more benign than if propensities depended freely on all patterns in $\bM$. Nevertheless, results from parametric models can be sensitive to the details of model specification---something that can partially be addressed by robustness checks varying key parameters but cannot be conclusively settled. 

\subsection{Interpretation in Satellite Image-based Observational Inference}
Having discussed identification, we now turn to important questions around estimation and interpretation in satellite image-based observational inference. 

With tabular data, interpreting results is generally straightforward for two reasons. First, features in a tabular dataset are human interpretable: we have measurements on pre-treatment variables such as age and ethnicity in a tabular dataset, and those quantities can be readily communicated via language. 

Second, linear models predominate in much of observational inference---both in weighting methods (which often involve a logistic regression step) or in methods modeling the outcome (where OLS is often applied). Linear models have a particularly simple structure where the relationship between the inputs and the predictions can be conveyed simply via regression coefficients. For example, in the OLS context (with covariate vector, $\bX$, and with no interaction/polynomial terms), we can readily communicate how one thing relates to another: 
\begin{align*}
\frac{\partial \; \widehat{\E}\left[Y \ \mid \ \bX \right]}{\partial \; X_{d}}= \beta_d. 
\end{align*} 
With linearity, gradients are the same for all values of $X_{d}$, leading to the interpretation of OLS coefficients: a one unit change in the predictor $X_{d}$ is associated with a $\beta_d$ unit change in the outcome, holding all else equal. Interpretation (in the sense of understanding how covariate inputs relate to model outputs) is arguably straightforward---at least if we sidestep the subtle issue of what it truly means to ``hold all else equal.''

In the image context, we can augment this derivative-based notion of interpretability via the use of salience maps \cite{gilpin2018explaining}. In particular, we can examine the following quantity in observational inference to explain {\it what} the prediction of the treatment assignment is particularly sensitive to in the image:  
\begin{align}
\mathbf{S} = \sqrt{  \sum_{c = 1}^C \left( \frac{\partial \; \widehat{\Pr}(T \ \mid \mathbf{M} )}{\partial \; \bM_{\cdot,\cdot,c}}\right)^2  }.
\end{align} 
$\mathbf{S}$ is of the same height and width dimensions as the raw image, $\bM$, and $\bM_{\cdot,\cdot,c}$ denotes the image slice of channel/feature $c$.\footnote{These features typically correspond to quantities such as reflectance, ultimately measured in energy per unit area per wavelength such as $\textrm{Watts}/(m^2 \cdot \mu m)$.}

The quantity, $\mathbf{S}$, represents the magnitude of the gradient of the predicted probability of treatment with respect to the image, where the magnitudes are computed across all the channel/feature dimensions of the image. This salience map, therefore, provides some indication of the parts of the image that are ``important'' in predicting the outcome, where that importance is calculated using derivatives to quantify importance as mathematical sensitivity. These derivatives are calculated via automatic differentiation, a computational tool to evaluate exact gradients even when no closed form is available \citep{griewank2003introduction}. Because they use automatic differentiation, salience maps in satellite-based observational inference can be computed with some, but not all, models for $\widehat{\Pr}(T \ \mid \bM)$. In particular, we require that $\widehat{\Pr}(T \ \mid \bM)$ is continuously differentiable with respect to $\bM$, which is the case for the most common class of image models such as convolutional networks (see Table \ref{tab:Differentiable} for a list of models outlined by differentiability). 

\begin{table}[ht]
\centering
\caption{Differentiability of  candidate models in image-based observational inference. Salience maps can be readily calculated for differentiable models.}\label{tab:Differentiable}
\begin{tabular}{ll}
\toprule \midrule
\underline{\it Differentiable} & Generalized linear models 
\\  & Feed-forward neural networks 
\\  & Convolutional models
\\  & Transformers 
\\  & Recurrent neural networks 
\\ & 
\\  \underline{\it Non-differentiable} &  Tree-based models 
\\ & Models involving greedy/discrete optimization 
\\ \bottomrule
\end{tabular}
\end{table}%

Another challenge related to interpretability in satellite image-based causal inference relates to the Stable Unit Treatment Value Assumption (SUTVA). This assumption states that the treatment of any scene $s$ should not affect the treatment status or outcome of another unit $s'$. In other words, each scene is i.i.d, as defined by our DAG. However, in spatial analysis, it may be harder to defend or even define SUTVA. Units that are closer in space may affect each other via spillover effects \citep{breza2016field}. For example, a policymaker allocating aid to one village may unintentionally affect outcomes in a nearby village, which benefits from the nearby allocation of assistance. SUTVA violations, and other forms of dependence (e.g., spatial clustering), can in principle be accounted for by specifying an appropriate variance-covariance structure \citep{sinclairDetectingSpilloverEffects2012}. However, this variance-covariance structure may be difficult to capture in practice, so that there may be difficulties in interpreting or explaining results from image-based observational causal inference. 

A more interesting possibility lies in the prospect that satellite image data may actually contain information about the latent interference structure present within the social system. As already noted, satellite data is informative transportation networks \citep{nagne2013transportation} which may correlate with patterns of social influence. Thus, while we can use block bootstrapping or other techniques to address spatial dependence, future work should probe the degree to which satellite data itself is informative about the underlying patterns of influence present within social reality. 

To conclude this section on interpretability in satellite image-based observational inference, we offer a few reflections on the centrality of resolution in this task. 

First, we note that resolution is a key driver of the residual confounding, as discussed in \S\ref{ss:identification}. We can only adjust for confounder objects in the image that can be resolved. Smaller confounder objects, therefore, introduce residual bias, indicating how technological improvements to sensor technology play a critical role in improving image-based causal inference methods. In short, the resolution determines the kinds of explanations we can make regarding our ability to reconstruct the treatment assignment mechanism. If there is a single pixel per scene, then that pixel can only capture global information about things such as the abundance of greenspace, soil moisture, and other quantities. If there are hundreds or thousands of pixels per scene, then more complex objects can be detected such as houses, roads, and trees. Interpretability in image-based inference is therefore inextricably tied to resolution. 

Another related motivation for obtaining high-resolution remote sensing data lies in the possibility that the use of satellite images can reduce researcher assumptions about how information from smaller scales aggregates up to the scene scale. Because each image is defined at a more granular resolution than the unit of analysis, we can use it to potentially reconstruct some unobserved confounders by learning the function generating confounders from the image. 

For example, assume that researchers seek to analyze the effect of a village-level treatment. From a government census, they obtain mean income information for the village ($s$) and then perform an analysis assuming $Y_s(0), Y_s(1)\perp T_s|X_s$, where $X_s$ contains the income data, $T_s$ is the treatment, and $Y_s(t)$, the potential outcome under $t\in\{0,1\}$. Yet, unless the mean income is the true confounder, the analysis will still be biased, which would occur if, in fact, \emph{minimum} income drove the decision to allocate $T_s$. However, using satellite images for each scene, we seek to reconstruct the minimum income signal based on our access to the higher-resolution data. 

In other words, we can weaken the assumption used in many empirical analyses that the variables measured at the scene level in fact contain the true confounders, when in reality highly non-linear functions---that use more granular information---may have generated the confounding structure. With image data, we, in principle, can hope to reconstruct some of those factors using advances in image-based machine learning models effective in the prediction domain (e.g., \citet{sun2013deep}).

Despite the potential of learning about how lower-level information aggregates up using high-resolution satellite images, there are still other subtleties to consider. For example, if treatments are defined at a high level of resolution (say, a house), but satellite data is defined at a lower level of resolution (say, a neighborhood block), there are ambiguities in how observational inference should be performed. Without higher-resolution data, we cannot adjust for home-level confounders, but if we aggregate treatments to the block level, information is lost, and extra researcher degrees of freedom are introduced in terms of how that aggregation should be done (whether in terms of summing, averaging, or aerial interpolating to the block level). For pixel-level treatments, therefore, extra caution is warranted due to the nuances of how resolution affects the confounders that can be captured by an image model using satellite resources. 

\subsection{Other Challenges and Opportunities of Satellite Image-based Observational Inference} \label{ss:estimation}

In addition to model dependence of observational inference with satellite images, there is another, perhaps more fundamental challenge. By its nature, satellite data is geo-referenced, and data on treatment assignments need to be geo-referenced as well if it is to be integrated into this pipeline. This linkage in some circumstances may be straightforward: a town receives an anti-poverty intervention, or it does not, and that town can be geo-referenced using APIs such as Google Maps and OpenStreetMap \citep{yeboah2021analysis}. However, in other circumstances, the linkage may be ambiguous: if regions are the units receiving treatment, it is less clear what satellite information should be used in modeling the treatment mechanism. In some situations, units are not geo-referenced at all: there may be data privacy or other reasons why even the approximate location of experimental units is explicitly not measured by researchers. Finally, it is important to note that, if disadvantaged units are less likely to be geo-referenced to neighborhoods than more advantaged units, we will again be in a situation where systematic missingness patterns bias causal estimates. These important limitations can also present opportunities for future research.

\section{Experiments Illustrating Observational Causal Inference Dynamics Under Model Misspecification and Varying Resolution}\label{ss:experiments}
Although the identification results described in \S\ref{ss:identification} are general, they are also theoretical and, as described in \S\ref{ss:estimation}, there are several practical considerations that will affect performance in real data. In order to better understand these dynamics, we use simulation, since, in practice. true causal targets are unknown. In these simulations, the image data is observed, but the confounding features are not, and must be estimated, as is the case in real applications. 

\subsection{Simulation Design} 
The simulation design centers on the scene level of analysis because, in practice, there seem to be fewer situations where pixels (14.25 meter by 14.25 meter grid cells) are treated than the situation where larger areas are treated (or where units are treated and we seek to adjust for confounding given their entire neighborhood context). 

In the simulation, confounding is generated by (1) applying a single convolution to our set of Landsat images from Nigeria (operation by $f_l(\cdot)$), (2) finding the maximum similarity across an image to the kernel filter (i.e. using the global maximum pooling operation), and (3) normalizing across observations so that the resulting confounder has mean 0, standard deviation 1. The confounder then enters equations for the treatment and outcome: 
\begin{align*} 
U_{swh} &= f_l(\bM_{s\Pi_s(wh)}) 
\\ U_{s} &= \textrm{GN}(\max\{ U_{swh}: {wh \in \Pi_{s}} \}),
\\ \Pr(T_{s} | \bM_s) &= \textrm{logistic}( \beta  U_{s} + \epsilon_{s}^{(W)})
\\ 
Y_{s} &= \gamma    U_{s} + \tau \, T_{s} + \epsilon_{s}^{(Y)}, 
\end{align*} 
where GN($\cdot$) denotes normalization to mean 0 and variance 1, done to prevent all units from receiving treatment. $f_l(\cdot)$ denotes the linear kernel function, where the single kernel filter, $l$, is a diagonal matrix and is shown in Figure \ref{fig:conv}. The set of images used in the simulation are drawn from the corpus of Landsat satellite images that we later use in the application (see \S\ref{ss:data}). The $\epsilon$'s are Normally distributed random error terms. 

We vary the dimensions of the estimating convolutional filter, $l$, keeping the structure of the true confounder-generating filter pattern fixed with a width of 9 (see Figure \ref{fig:conv}). By varying the estimating filter dimensions, we alter the size of the neighborhood used in estimating the latent confounder, allowing us to probe model misspecification, where there is a gap between the true data-generating process and estimating model structure. 

We also vary the resolution of the image used in estimation. This quantity varies from 1 (when the observed image resolution is the same as used in confounder generation) to 0.12 (where the image resolution observed in the estimation step is 12\% that of the original; in other words, pixels are 88\% larger in width). By varying the resolution, we can probe how this unique feature of images affects observational causal inference estimation in practice.

We compare two estimators. First, we examine the difference-in-means estimator, $\hat{\tau}_{\textrm{0}}$, defined as the difference in conditional outcome means for the two treatment groups. Because of confounding, this baseline quantity is biased for $\tau$. 

We also report results from an Inverse Propensity Weighting (IPW) estimator \citep{austin2015moving}, with image model estimation performed using a single-layer convolutional network model which is trained using stochastic gradient descent with the binary cross-entropy loss function (which is equivalent to the negative log-likelihood). The estimation formula for IPW is $\hat{\pi}(\bM_s)=\widehat{\textrm{Pr}}(T_s=1|\bM_s)$, $\widehat{\tau} =  \frac{1}{n}\sum_{s=1}^n \left\{ \frac{T_s Y_s}{\hat\pi(\bX_s)} -\frac{(1-T_s)Y_s}{1-\hat\pi(\bX_s)}\right\}$ for the scene analysis and is defined analogously at the pixel level. We report results from the Hajek IPW estimator where the weights have been normalized, reducing estimator variance at the cost of some finite sample bias \citep{skinner2017introduction}. 

In our evaluation, we compare the true $\tau$ with the estimated values. We compute the absolute bias of $\hat{\tau}$ and the Root Mean Squared Error (RMSE): 
\begin{equation}
\textrm{Absolute Bias} = \big|\E[\hat{\tau} - \tau]\big|; \;\;
\textrm{RMSE} = \sqrt{\E[(\hat{\tau} - \tau)^2 ]},
\end{equation}
where expectations are taken over the data-generating process and are approximated via Monte Carlo. These evaluation metrics are then normalized using the MSE from the baseline estimator, $\hat{\tau}_0$, to facilitate the interpretability of the results (so that Relative Absolute Bias and Relative RMSE are reported).

\subsection{Simulation Results}
In Figure~\ref{fig:SceneResultsVaryGlobality}, we show the dynamics of satellite image-based observational inference in this simulation. As expected, we find that the absolute bias and RMSE are minimized when the kernel width used in estimation is the same as the kernel width used to generate the true confounder and when the resolution is the same as that used in generating the confounder. The bias and RMSE grow larger when the estimating kernel width is greater than the kernel width used in confounder generation. This finding is likely due to the fact that, when the neighborhood size used in estimation is larger than that used to create the confounder, the additional parameters allow the model to overfit. The bias and RMSE panels of Figure \ref{fig:SceneResultsVaryGlobality} look similar because the variance of estimation is relatively small, especially when the resolution has been down-shifted. 

Resolution also plays an important role in explaining the dynamics of observational inference with satellite images. For all values of resolution, bias and RMSE of estimation are still  favorable compared to using the naive difference-in-means estimator without satellite-based neighborhood information. In other words, low-resolution images still improve upon the simple difference-in-means estimator of the treatment effect. 

When the resolution of the image used in estimation is the same as that that actually generated the confounder, performance in terms of bias and RMSE is optimal when model specification is correct, but, with this higher resolution image, the analysis is also more sensitive to model misspecification. When the resolution is down-shifted, model misspecification matters less; there is less image information available so the dependence of the results on the image model is weakened. We find similar results when we vary the degree of noise in the scene-level confounders, loosening the determinism assumption discussed above (see \S\ref{ss:AdditionalSims}).  
\begin{figure}[t]
    \centering
        \centering
       \includegraphics[width=0.90\linewidth]{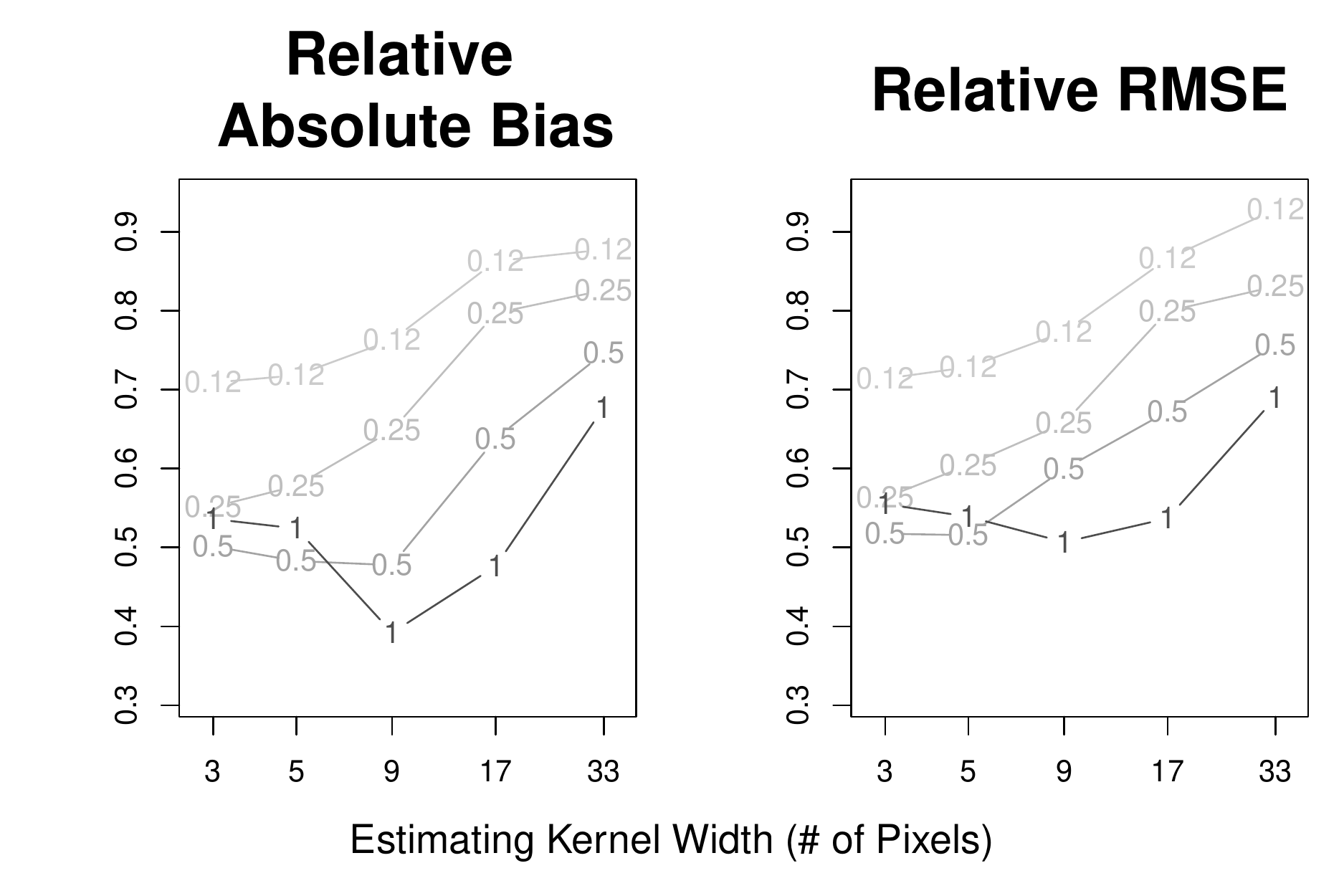}
        \caption{Scene-level results varying the resolution and character of model misspecification. The true confounder was generated with a kernel width of 9 and resolution scaling factor of 1. Resolution scaling factors are numerically labeled and colored with grayscale. A resolution scaling factor of 0.5, for example, indicates that the pixels in that analysis are twice as wide as in the original raw image.}\label{fig:SceneResultsVaryGlobality}
\end{figure}

\section{Empirical Illustration: The Impact of Local Aid Programs on Neighborhood-level Poverty in Nigeria}\label{sec:mu_empirical}
Having examined some of the opportunities and pitfalls of satellite-based observational inference, in this section, we demonstrate it in the context of Nigeria, Africa's largest economy and a country that is projected to be the world's second most populous by 2100 \citep{vollset2020fertility}. This context is one in which identifying effective anti-poverty interventions carries substantial practical importance: despite an average economic growth rate of around 3\% since 2000, about 40 percent of the Nigerian population live below the poverty line (\$2 per day). In response, governments and NGOs have deployed a variety of local aid programs to the country. However, the causal impact of these programs is difficult to estimate \citep{roodman2008through}. 

While geo-temporal data on poverty, $Y$, and interventions, $T$, are readily available, there is a lack of geo-temporal data on potential confounders, $U$, at the local level \citep{daoud_impact_2017,hallerod2013bad}. While some of these confounders may be difficult to capture directly using images (such as the quality of political institutions), there may be information present in remote sensing imagery about other confounder objects related to infrastructure or agriculture \citep{schnebele2015review,steven2013applications}. We therefore use satellite images of Nigerian communities in order to estimate the impact of local aid programs. 

\subsection{Data Description}\label{ss:data}
Our outcome data on poverty is drawn from the Demographic and Health Surveys (DHS), which are conducted by a non-profit organization, ICF International, with funding from USAID, WHO, and other international organizations \citep{rutstein2006guide}. The DHS surveys are conducted at the household level for randomly selected clusters that aggregate to geographically explicit scenes of about 300 households for de-identification purposes. Our outcome measure is the International Wealth Index (IWI) from 2018, which is a principal-components-derived summary of 12 variables including households' access to public services and possession of consumer products such as a phone. Its scale is between 0 and 100, with higher values indicating greater wealth. 

Our treatment data is drawn from a data set on international aid programs to Nigeria compiled by AidData \citep{aiddata2016} used under an Open Data Commons License. The aid programs we examine took place after 2003. The aid providers include entities such as the World Bank and WHO, as well as domestic governments such as the United States. The programs we examine are diverse, focusing on infrastructure (e.g., support for road development) and agriculture (e.g., support for small-scale farmers), among other things. For the simplicity of our presentation, we take the presence of a geographically-specific aid program within 7000 meters of a DHS point as our treatment. 

Our pre-treatment image data is drawn from Landsat, the satellite imagery program operated by NASA/USGS. We use the Orthorectified ETM+ pan-sharpened data derived from the raw satellite imagery captured between 1998 and 2001; the raw data have been processed to contain minimal cloud cover and to be correctly geo-referenced. Resolution is 14.25 meters; reflectance in the green, near-infrared, and short-wave infrared bands is recorded. Treatment and control locations are shown in the left panel of Figure \ref{fig:NigeriaContext}.

\begin{figure}[h]
    \centering
        \centering
        \includegraphics[width=0.75\linewidth]{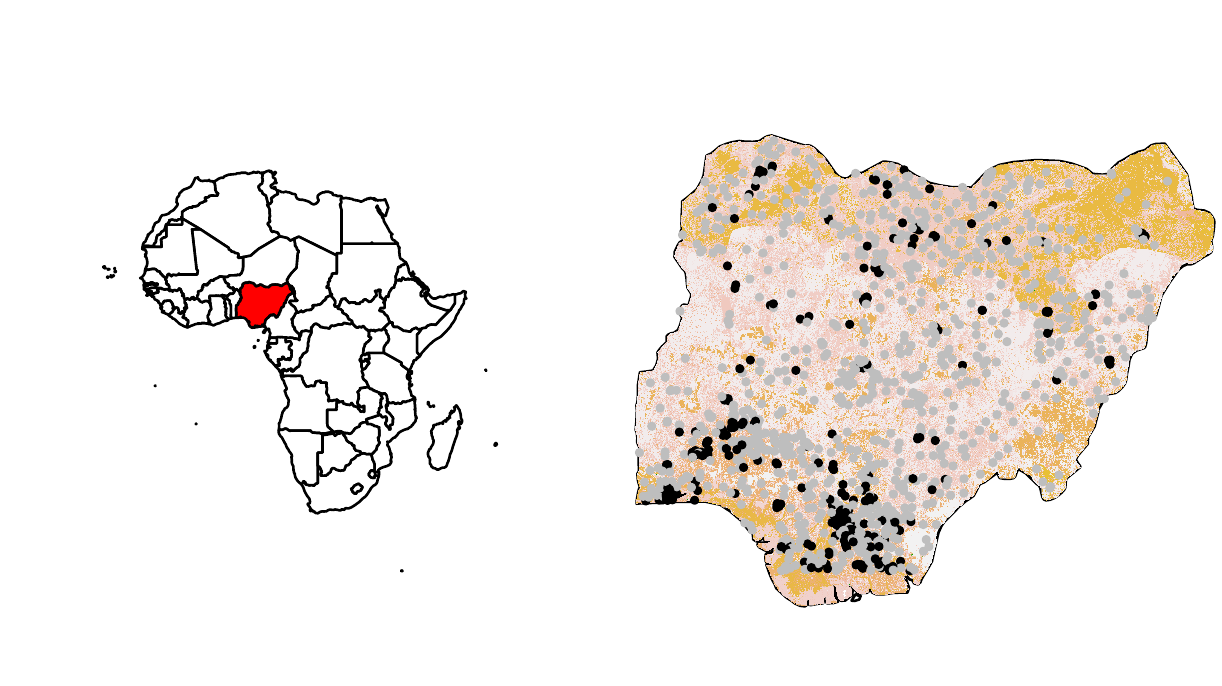}
        \caption{The left panel identifies Nigeria as the context of interest in this observational study. The right panel illustrates the location of treatment and control sites. Gray points are control locations; black points are treatment locations. }\label{fig:NigeriaContext}
\end{figure}

\subsection{Observational Inference Image Model Design}\label{ss:model}
Our image model for the treatment is built using three convolutional layers (32 filters each) with max pooling. After the application of the filters, we project the channel dimension into a 3-dimensional space to facilitate interpretability via a single image post-convolution. Batch normalization is used on the channel dimensions and before the final projection layer. Weights are learned using ADAM with Nesterov momentum with cosine learning rate decay (baseline rate = 0.005; \citet{gotmare2018closer}). We compare performance across a variety of filter sizes. We attempt to limit overfitting by randomly reflecting each image dimension with probability $1/2$ during training. We assess out-of-sample performance using three random training/test splits, averaging over this sampling process and reporting 95\% confidence intervals from the three test set assessments. We vary the image model structure by altering the convolutional filter width (which affects the size of the image patterns analyzed in the image model). 

We also assess stability of the results to resolution of the satellite images used. This task contains some subtleties. It is always possible to lower the resolution of an image, averaging across pixel grids. However, when we change the dimensionality of the input image, we may no longer be able to employ the same image models as in the full dimensionality case. The reason for this is due to the fact that each convolutional layer in the image model reduces the output dimensionality; if our starting resolution is too small, the implied output dimensionality of some of the convolutional layers would be negative. We keep the image model structure the same and analyze results across different resolutions. As a consequence, when the image model cannot be fit with a given kernel width and resolution combination, we drop that model from the set of models analyzed. 

\subsection{Empirical Results}\label{ss:ApplicationResults}
First, in the left panel of Figure \ref{fig:ApplicationMetrics}, we assess the propensity model fit to the treatment data. We find that the image model always improves on the baseline out-of-sample loss value obtained by random guessing of the dominant class (the control class, comprising 71\% of the data sample). Performance is stable across values of the kernel width used in estimation. Moreover, resolution does not greatly affect performance. This stability implies that more macroscopic aspects of the landscape are ultimately what is driving the estimated treatment assignment mechanism, as opposed to the lower-level features such as the grouping of houses of a certain type. 

\begin{figure}[t]
    \centering
    \includegraphics[width=1.\textwidth]{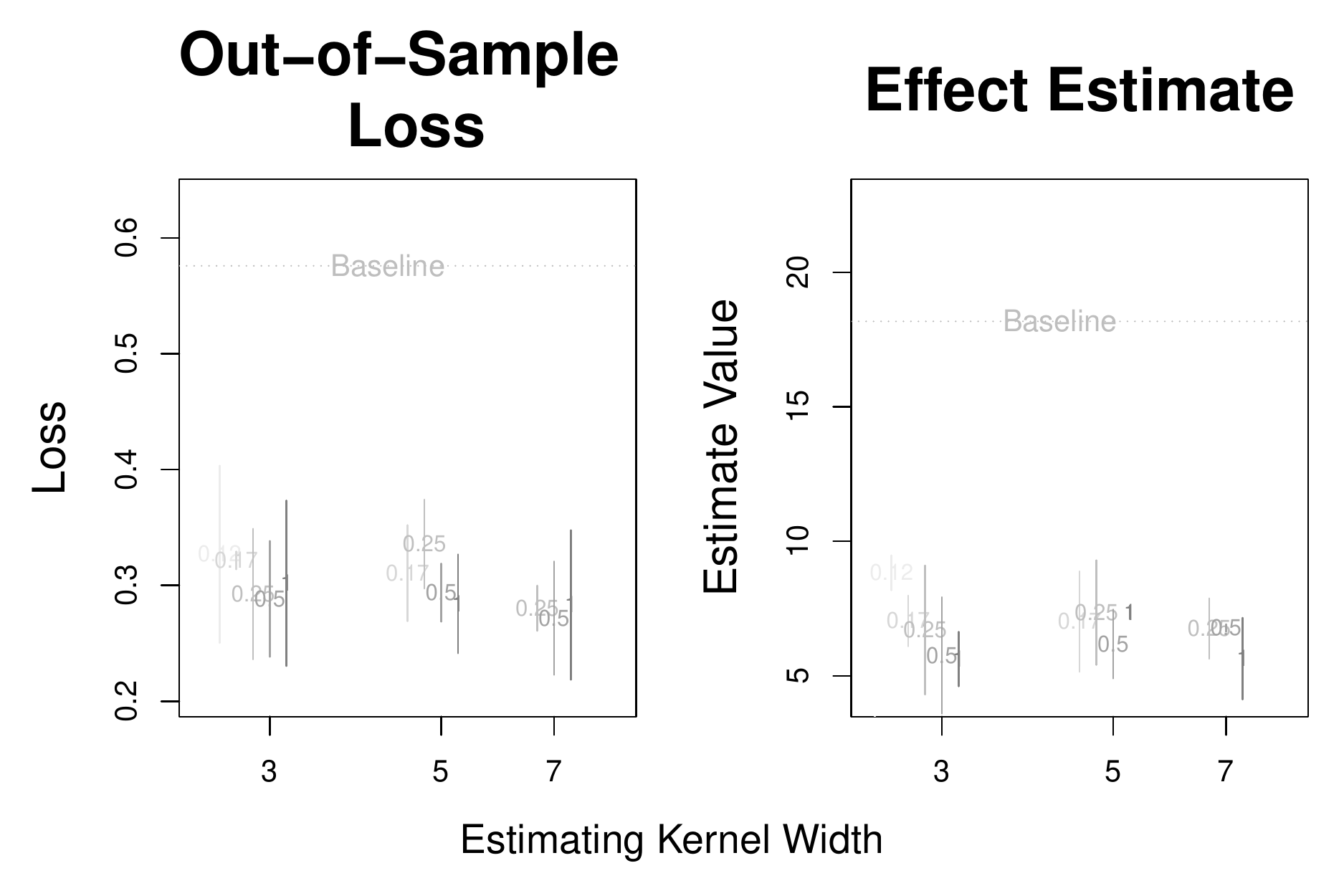}
    \caption{The left panel shows out-of-sample binary cross entropy loss compared to the baseline loss when guessing the dominant class. The right panel shows the IPW estimate values across the range of estimating kernel width values compared to the baseline difference in conditional means.}\label{fig:ApplicationMetrics}  
\end{figure} 

Next, in the right panel of Figure \ref{fig:ApplicationMetrics}, we analyze the estimated treatment effects. We find that, across the estimating kernel width range, adjusted estimates are positive but smaller in magnitude than the baseline difference in means. This hints at the importance of confounder adjustment, as these programs may be given to areas already primed for growth. 

We also analyze the estimation model dynamics. In Figure \ref{fig:InsideConvnet}, we visualize data for the three out-of-sample control (left) and treated (right) units. We also display the salience maps for the predicted treatment probabilities as well as the output from the final spatially resolved layer in the image model when the kernel width is 7. To explore robustness, we also display in \S\ref{ss:EmpiricalResultsRobust} results from another run with a kernel of 7 in the estimation model, as well as width kernel widths 3 and 5. 

To take examples from the model output, a particularly low treatment probability site is estimated in the remote desert city of Machina (pop. 62,000); a high probability site is from the city of Katsina (pop. 429,000) near a large agricultural basin and with rich water resources (water soil content is displayed as red in the RBG satellite images using non-visible band). The output of this model would seem to resonate with the fact that many of the projects undertaken by global actors are specifically designed to assist farmers and agriculture more generally. 

\begin{figure}[t]
    \centering \includegraphics[width=1.\textwidth]{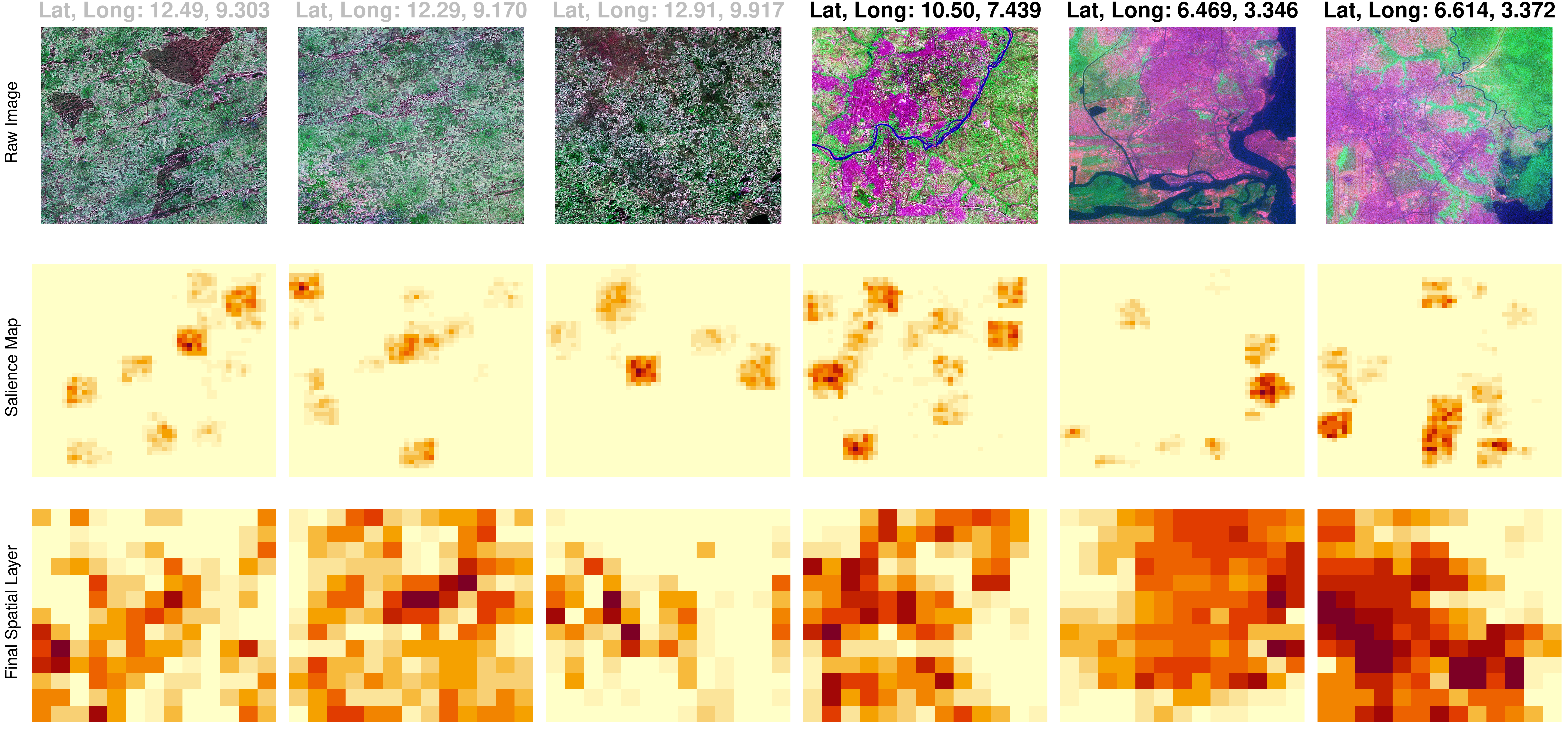}
    \caption{The three left panels depict the raw data for control units, salience maps for the predicted treatment probability, and output from the final spatially resolved layer in the image model. The three right panels depict the same things for treated units. Note that the three bands of the satellite image are not ``red'', ``green'', and ``blue'' (see main text), so in this visual representation, reddish pink indicates soil moisture content. }\label{fig:InsideConvnet}
\end{figure} 

Finally, we assess the performance of the estimated propensity scores on reducing covariate imbalance between treated and control groups. The same absence of rich covariate information in places such as sub-Saharan Africa that motivates this paper also makes this assessment task difficult. Still, we can analyze differences in longitude/latitude between treated/control groups. We find a raw difference of (-1.12, -1.13). After weighting, this difference decreases in magnitude to (-0.39,0.53), indicative of improved counterfactual comparisons. 

Overall, this application shows some of the nuances of satellite image-based observational inference---how resolution affects the space of possible image models, how to think about interpreting the output of the causal inference model in the image context, as well as how to validate the improvement in counterfactual comparisons after employing the satellite information.

\section{Discussion and Future Work}\label{s:discussion}
In this paper, we characterize key challenges and opportunities of satellite image-based observational causal inference. We formally show that, even though confounder aspects of a neighborhood may be latent, we can adjust for them using the image information. There are several subtleties, however, related to resolution and the degree to which the image truly proxies those confounders. We illustrate some of these tradeoffs in simulation. We also apply the satellite image-based observational inference approach in order to understand the causal effect of aid programs on reducing neighborhood-level poverty in Africa's largest economy. As previously stated, our approach is not limited to poverty. It can be used for analyzing the effects of environmental factors \citep{shiba_heterogeneity_2021,daoud_what_2016} to neighborhood-level change \citep{lin_what_nodate}, and beyond.

The use of satellite images in observational inference approach has some inherent limitations. For example, as described above, confounding may be due to objects that cannot be resolved in the image data, and, as a result, bias reduction will not occur conditioning on the image information. In this context, the collection and application of imagery with higher spatial, temporal, and spectral resolution is a priority. Spatial and temporal resolution may both be achieved, for example, using sensors mounted on ground-based infrastructure (e.g., \citet{johnston2021measuring}); spatial resolution and extent could be optimized with drone- or airplane-based instruments (e.g., \citet{gray2018integrating}). Future considerations should examine privacy and fairness issues with causal analyses based on passive sensor technologies. 

Second, while in some cases, such as if disparate individuals are treated, the scene-level unit of analysis is clearly defined (e.g., the individual within a neighborhood context), in other contexts, the scene-level unit of analysis is more ambiguous. The researcher therefore has choices about how to define the scene (e.g., at the street, village, or region levels), a choice that could introduce systematic bias into the analysis. This issue is known as the modifiable areal unit problem \citep{fotheringham1991modifiable}, and the approach described here is vulnerable to it as well. Future work should therefore focus on the development of image-confounding methods that have theoretical guarantees on the robustness of the results to the scale of action examined. This issue of scale is also related to the question of capturing the treatment and outcome at different scales.

Third, machine-learning-based image models learn the image patterns that best predict treatment assignment. More research is needed to connect these patterns, learned inductive, with the mental processes of real actors as they consult images in decision-making. This path of research could forge links between cognitive science, machine learning, text analysis, and causal inference.  \hfill $\square$

\medskip

\bibliographystyle{plainnat}
\bibliography{confoundingbib}

\clearpage\newpage
\setcounter{page}{1}
\appendix

\renewcommand{\thefigure}{A.\arabic{figure}}
\setcounter{figure}{0}

\section{Appendix}
\subsection{Connections with the Causal Proxy Literature}\label{ss:proxy}
Besides facilitating the use of causal inference for a social science audience, our work is  related to the literature on identification via proxies~\citep{tchetgen2020introduction} or drivers~\citep{pearl2013linear} of confounders. For the former, \citeauthor{louizosCausalEffectInference2017a}, developed Causal Effect Variational Autoencoder (CEVAE) which uses proxies to infer the distribution of the latent confounder and use this in  adjustment. In contrast, our approach adjusts for an observed variable---the image. We formalize key assumptions required for the correctness of this method and provided a general framework for conducting causal inference using images, where unlabeled objects in the image may affect both treatment and outcome \citep{castroCausalityMattersMedical2020}. This image-based confounding bias might in some circumstances be equivalent to traditional spatial interdependence, but differs insofar as the confounding bias is defined with reference to unlabeled entities in the image, thereby injecting bias \citep{paciorek2010importance}. Relying on our formalization and model implementation, we analyze aid interventions (treatment) and poverty (outcome) in Africa---something of policy relevance as policymakers often rely on satellite images for aid intervention \citep{voigtGlobalTrendsSatellitebased2016, bediMorePrettyPicture2007}.

\subsection{Additional Simulation Results}\label{ss:AdditionalSims}

\subsubsection{Probing Estimation Bias as the Determinism Assumption is Relaxed}
We here explore how model misspecification affects estimation error. In particular, we probe how relaxing the determinism assumption of Proposition \ref{prop:identification} affects satellite-based observational inference. In particular, we now let the unobserved confounder be a random function of the satellite image, as depicted visually in Figure \ref{fig:ResidualBias}. In particular, the confounder values are now 
\begin{equation}
U_{swh} = f_l(\bM_{s\Pi_s(wh)}) + \epsilon_{swh}^{(U)},
\end{equation}
where $\epsilon_{swh}^{(U)}
\sim \mathcal{N}(0,\sigma_U^2)$. We vary $\sigma_U^2 \in \{1,3,5,7\}$. We then apply the same data-generating process to obtain scene-level treatments and outcomes. 

We see in Figure \ref{fig:SceneResultsVaryGlobalityNoise} how performance is affected by relaxing the determinism assumption. As expected, we see that estimation bias grows as the unobserved confounding is increasingly determined by the noise factor, $\epsilon_{swh}^{(U)}$. When the noise scale is at its maximum, bias is still no worse than the simple difference in means baseline (i.e., relative bias/RMSE approaches 1). This fact is likely because the noise injected into the confounding mechanism is itself exogenous. Nevertheless, having established a theoretical baseline in this paper, future research should examine this noise-induced confounding to image-based causal inference in greater detail, connecting this line of work with the proxy literature. 

\begin{figure}[h]
    \centering
        \centering
        \includegraphics[width=0.75\linewidth]{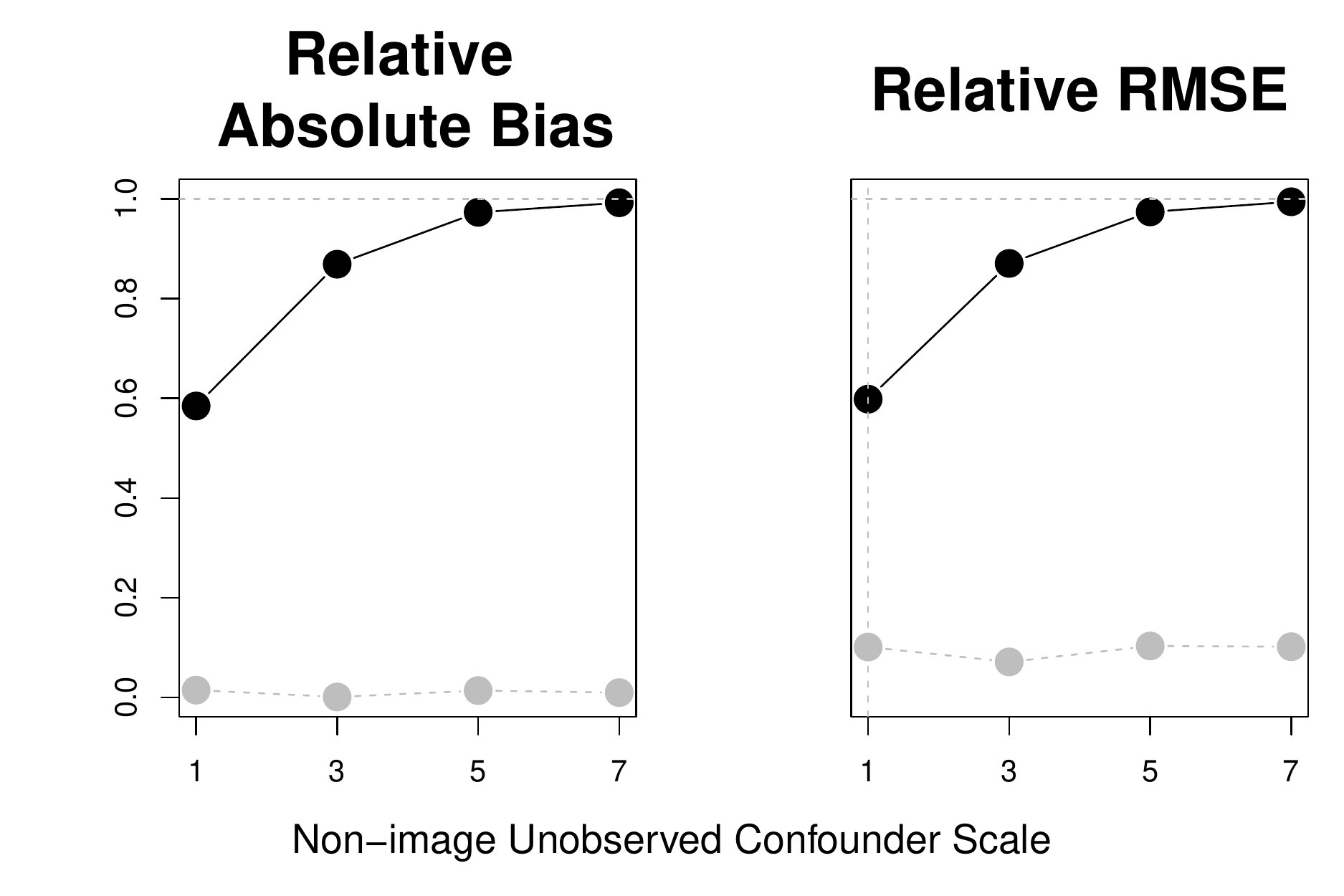}
        \caption{Bias and RMSE for the scene-level analysis as we vary the stochasticity present in the confounding mechanism, holding the estimation kernel width fixed at 8. Gray circles indicate effect estimate values using the true (in practice unobserved) treatment probabilities. }\label{fig:SceneResultsVaryGlobalityNoise}
\end{figure}

\newpage
\subsection{Robustness to Model Specification in the Empirical Results}\label{ss:EmpiricalResultsRobust}
\begin{figure}[h]
    \centering \includegraphics[width=1.\textwidth]{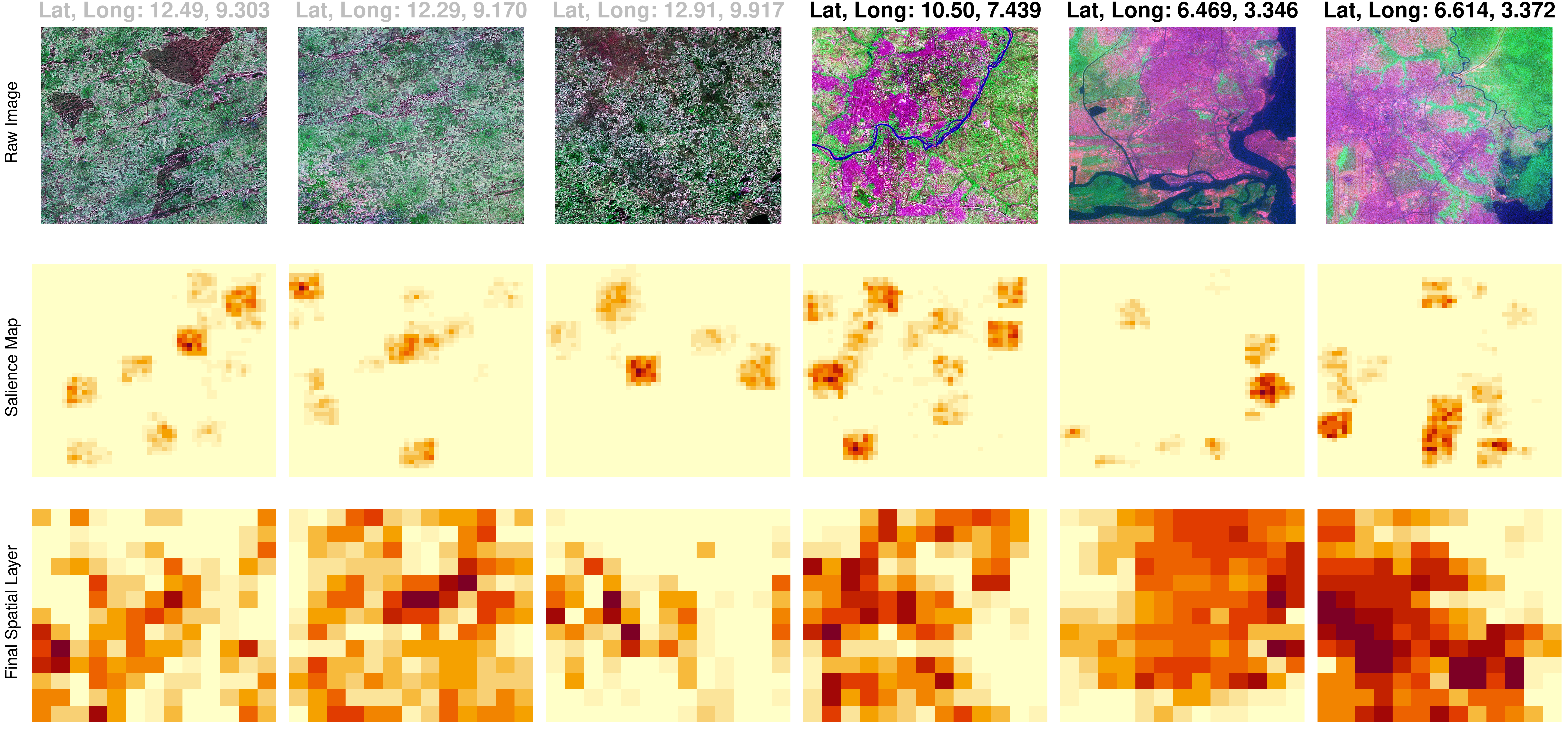}
    \caption{Replicating Figure \ref{fig:InsideConvnet} with another training/test split.}
\end{figure} 

\begin{figure}[h]
    \centering \includegraphics[width=1.\textwidth]{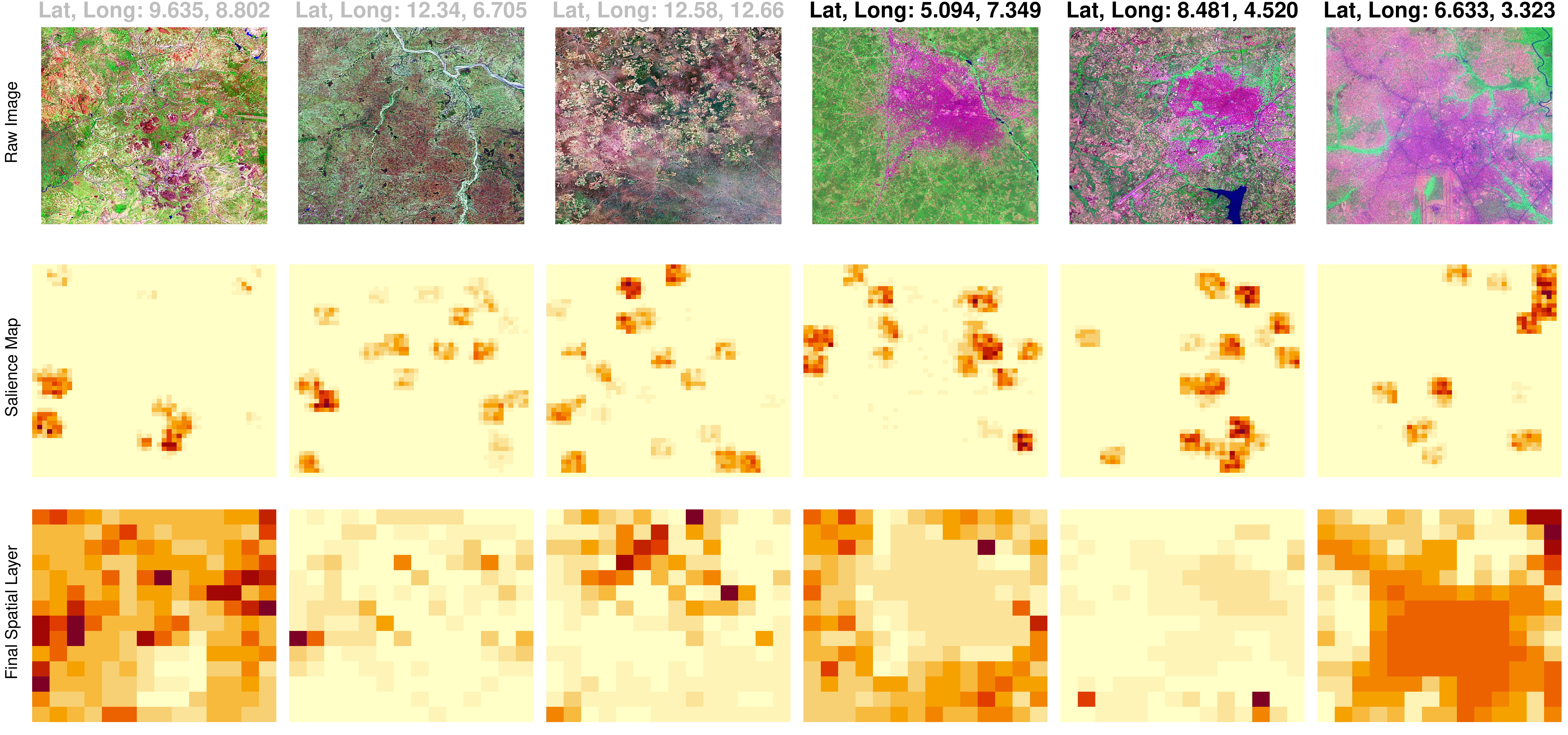}
    \caption{Replicating Figure \ref{fig:InsideConvnet} with an estimating kernel width of 5 instead of 7.}
\end{figure} 

\begin{figure}[h]
    \centering \includegraphics[width=1.\textwidth]{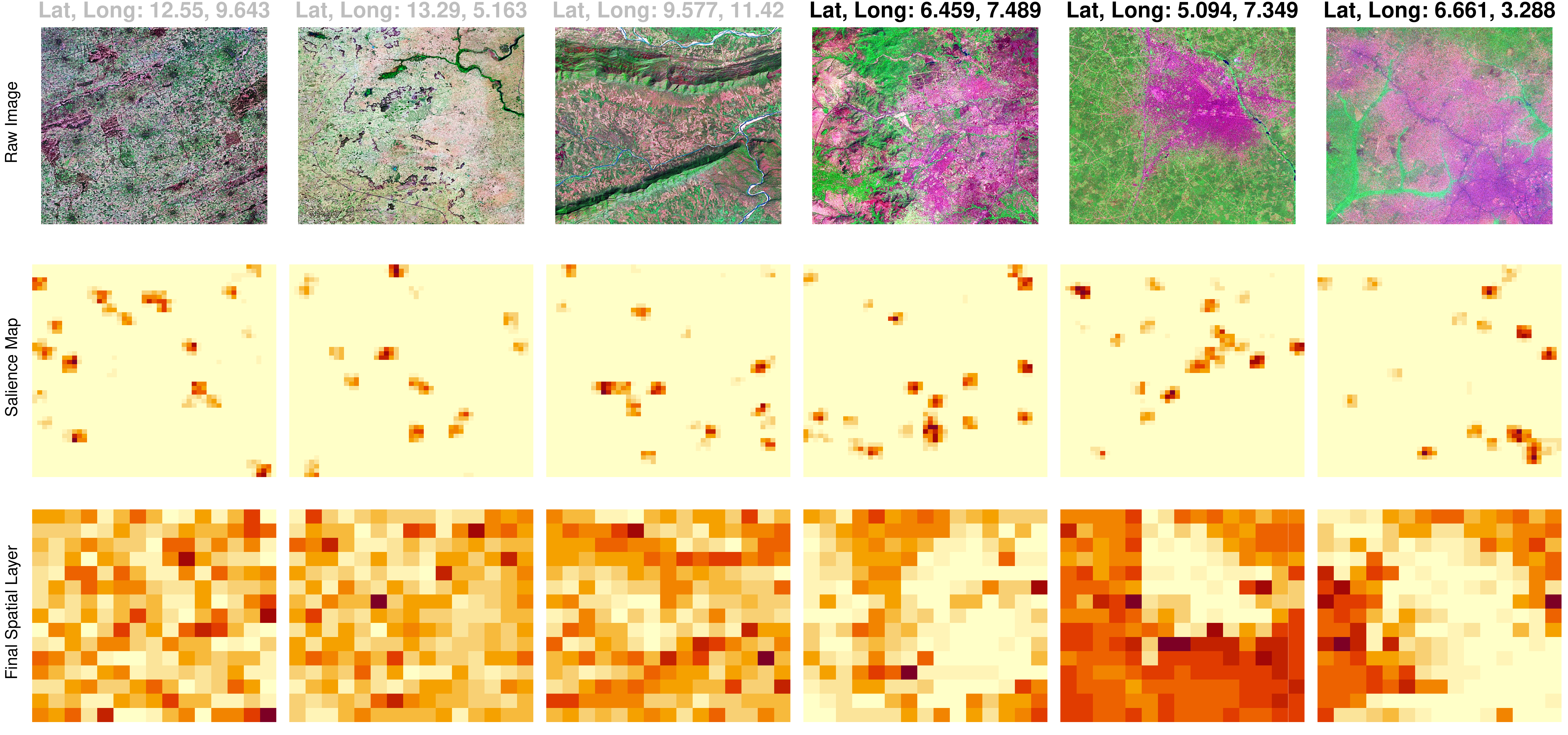}
    \caption{Replicating Figure \ref{fig:InsideConvnet} with an estimating kernel width of 3 instead of 7.}
\end{figure} 

\clearpage 

\subsection{Implementation Details}
We implement our computational analyses on an Apple M1 GPU using Metal-optimized TensorFlow 2.11 with GNU Parallel. Total compute time for the simulations is about 48 hours; total compute time for the application results is about 24 hours. 

\end{document}